\documentclass[sigconf]{acmart}
\usepackage{mathtools}
\usepackage{dsfont}
\usepackage{bm}
\usepackage{algorithm}
\usepackage{algpseudocode}
\usepackage{amsmath}
\usepackage{amsthm}
\usepackage{amsfonts}
\usepackage{multirow}
\usepackage{enumitem}
\usepackage{romannum}
\AtBeginDocument{%
  \providecommand\BibTeX{{%
    \normalfont B\kern-0.5em{\scshape i\kern-0.25em b}\kern-0.8em\TeX}}}

\copyrightyear{2023}
\acmYear{2023}
\setcopyright{rightsretained}
\acmConference[CIKM '23]{Proceedings of the 32nd ACM International Conference on Information and Knowledge Management}{October 21--25, 2023}{Birmingham, United Kingdom}
\acmBooktitle{Proceedings of the 32nd ACM International Conference on Information and Knowledge Management (CIKM '23), October 21--25, 2023, Birmingham, United Kingdom}\acmDOI{10.1145/3583780.3614885}
\acmISBN{979-8-4007-0124-5/23/10}

\newcommand{\shortname}{CEMSP}
\newcommand{\fullname}{Counterfactual Explanations with Minimal Satisfiable Perturbation (CEMSP) }

\newtheorem{theorem}{Theorem}

\newtheorem{prop}{Proposition}
\newtheorem{definition}{Definition}

\newtheoremstyle{case}{}{}{}{}{}{:}{ }{}
\theoremstyle{case}
\newtheorem{case}{Case}

\numberwithin{subcase}{case}

\makeatletter
\gdef\@copyrightpermission{
   \begin{minipage}{0.3\columnwidth}
     \href{https://creativecommons.org/licenses/by-nc-sa/4.0/}{\includegraphics[width=0.90\textwidth]{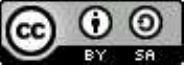}}
   \end{minipage}\hfill
   \begin{minipage}{0.7\columnwidth}
     \href{https://creativecommons.org/licenses/by-nc-sa/4.0/}{This work is licensed under a Creative Commons Attribution-NonCommercial-ShareAlike International 4.0 License.}
   \end{minipage}
   \vspace{5pt}
}
\makeatother

\begin{document}

\title{Flexible and Robust Counterfactual Explanations with Minimal Satisfiable Perturbations}

\author{Yongjie Wang}
\orcid{0000-0003-4718-0742}
\email{yongjie002@e.ntu.edu.sg}
\affiliation{%
  \institution{Nanyang Technological University}
  \country{Singapore}
}

\author{Hangwei Qian}
\orcid{0000-0003-4831-0748}
\email{qian_hangwei@cfar.a-star.edu.sg}
\affiliation{%
  \institution{A*STAR}
  \country{Singapore}
}

\author{Yongjie Liu}
\orcid{0009-0007-7935-8601}
\email{202215223@mail.sdu.edu.cn}
\affiliation{%
  \institution{Shandong University}
  \country{China}
}

\author{Wei Guo}
\orcid{0000-0002-8124-5186}
\email{guowei@sdu.edu.cn}
\affiliation{%
  \institution{Shandong University}
  \country{China}
}

\author{Chunyan Miao}
\orcid{0000-0002-0300-3448}
\email{ASCYMiao@ntu.edu.sg}
\affiliation{%
  \institution{Nanyang Technological University}
  \country{Singapore}
}

\renewcommand{\shortauthors}{Yongjie Wang, Hangwei Qian, Yongjie Liu, Wei Guo, \& Chunyan Miao}

\begin{abstract}
Counterfactual explanations (CFEs) exemplify how to minimally modify a feature vector to achieve a different prediction for an instance. CFEs can enhance informational fairness and trustworthiness, and provide suggestions for users who receive adverse predictions. However, recent research has shown that multiple CFEs can be offered for the same instance or instances with slight differences. Multiple CFEs provide flexible choices and cover diverse desiderata for user selection. However, individual fairness and model reliability will be damaged if unstable CFEs with different costs are returned. Existing methods fail to exploit flexibility and address the concerns of non-robustness simultaneously. To address these issues, we propose a conceptually simple yet effective solution named \textit{Counterfactual Explanations with Minimal Satisfiable Perturbations (CEMSP)}. Specifically, \shortname\ constrains changing values of abnormal features with the help of their semantically meaningful normal ranges. For efficiency, we model the problem as a Boolean satisfiability problem to modify as few features as possible. Additionally, \shortname\  is a general framework and can easily accommodate more practical requirements, e.g., casualty and actionability. Compared to existing methods, we conduct comprehensive experiments on both synthetic and real-world datasets to demonstrate that our method provides more robust explanations while preserving flexibility.  
\end{abstract}

\ccsdesc[500]{Computing methodologies~Philosophical/theoretical foundations of artificial intelligence}
\ccsdesc[500]{Computing methodologies~Symbolic and algebraic manipulation}
\ccsdesc[500]{Theory of computation~Logic and verification}
\ccsdesc[500]{Computing methodologies~Knowledge representation and reasoning}

\keywords{Counterfactual explanations, multiplicity, normal ranges, flexibility, robustness}

\maketitle

\section{introduction}
\label{sec:introduction}
Understanding the internal mechanisms behind model predictions is difficult due to the large volume of parameters in machine learning models. This problem is particularly significant in high-stakes domains such as healthcare and finance, where incorrect predictions are disastrous \cite{daws-2021}. Counterfactual explanation (CFE) \cite{s2017counterfactual} aims to identify minimal changes required to modify the input to achieve a desired prediction and provides insights into why a model produces a certain prediction instead of the desired one. CFEs can help understand the underlying logic of certain predictions \cite{10.1145/3531146.3533218}, detect the inherent model bias for fairness \cite{DBLP:conf/fat/KasirzadehS21}, and provide suggestions to users who receive adverse predictions \cite{DBLP:conf/fat/KarimiSV21,ustun2019actionable}. Therefore, CFEs can be adopted in broad applications of healthcare, finance, education, justice, and other domains. 

Despite the valuable insights provided by CFEs, recent studies \cite{s2017counterfactual,10.1145/3531146.3533218,sokol2019counterfactual,10.1016/j.artint.2022.103840,slack2021counterfactual,laugel2019issues} have shown that multiple CFEs can exist with equivalent evaluation metrics (e.g., validity, proximity, sparsity, plausibility), yet significantly differ on feature values for an input or seemingly indifferent inputs. For example, ~\citeauthor{s2017counterfactual} \cite{s2017counterfactual} pointed out ``multiple counterfactuals are possible, ...'' and ``multiple outcomes based on changes to multiple variables may be possible''. ~\citeauthor{laugel2019issues} \cite{laugel2019issues} demonstrated that instances that are close to each other can have different CFEs. Moreover, ~\citeauthor{10.1016/j.artint.2022.103840} \cite{10.1016/j.artint.2022.103840} have argued that CFEs lack robustness to adverse perturbations if not deliberately designed. While multiple CFEs can lead to the same desired prediction, each CFE tells a different story to reach the target. 

It is important to note that counterfactual multiplicity has both advantages and disadvantages. On one hand, multiple CFEs can be beneficial because they afford more flexibility and freedom to select user-friendly CFEs when a single CFE may be  overly restricted for users. Specifically, diverse CFEs \cite{russell2019efficient,mothilal2020explaining} are offered to potentially cover broad user preferences; interactive human-computer interfaces \cite{9229232,10.1145/3459637.3482397} are designed on multiple CFEs to obtain more satisfiable ones. On the other hand, users who have the same feature values or seemingly inconsequential differences may receive inconsistent CFEs (e.g., two different diverse sets) as the CFE method itself does not store historical CFEs and guarantee the optimal solutions either. Such inconsistency inevitably raises fairness issues \cite{10.1016/j.artint.2022.103840,artelt2021evaluating} and undermines users' trust \cite{10.1145/3531146.3533218} in CFEs. For example, two financially similar individuals are rejected when they apply for a loan. Yet, CFEs for two people are quite different-one needs to update the salary slightly while the other is required to get a higher education degree and a better job. Another negative example is when users make some efforts towards previous CFEs but receive a significantly different CFE, rendering their previous efforts futile. Therefore, it is crucial to take advantage of the flexibility of multiple CFEs and maintain consistency for users having the same feature values or slight differences. 

Recent research \cite{upadhyay2021towards,black2022consistent,pawelczyk2020counterfactual,von2022fairness,dutta2022robust} on robustness mainly studies CFEs with consistent predictions under slight model updates (by restricting CFEs to preserve causal constraints, or follow data distribution, etc.), rather than generates CFEs with consistent feature values. Therefore, these studies fail to address fairness concerns and ignore the freedom of user selection.  Generally, models to explain are highly complex and non-convex, e.g., DNN models. Even constrained by consistent prediction, heuristic search strategies \cite{s2017counterfactual,laugel2017inverse,10.1145/3375627.3375812,tolomei2017interpretable} can still converge to different non-optimal solutions due to the huge search space. Meanwhile, these works do not explicitly exploit flexibility to meet user preferences. As the number of possible CFEs can be huge, existing methods on flexibility or robustness can be explained to be different selection strategies from a solution pool (without optimizing diversity and robustness in advance), i.e., selecting CFEs that are diverse \cite{mothilal2020explaining}, follow data distribution \cite{pawelczyk2020counterfactual}, or 
withhold causal constraints \cite{von2022fairness}. Motivated by this, we target to design a novel method that obtains a diverse and robust set of CFEs simultaneously.

To overcome the above limitations, we propose to incorporate task priors (normal ranges, a.k.a. reference intervals) to stabilize valid search regions, while ensuring that counterfactual explanations (CFEs) are diverse to meet various user requirements. It should be noted that robustness measures the differences between two sets of CFEs in different trials while diversity measures the inherent discrepancy within a set of CFEs. Normal ranges in our approach commonly exist in broad domains and are easy to obtain from prior knowledge. For example, the normal range of heart rate per minute is between 60 and 100; the IELTS score should be greater than 6.5 and the minimal GPA is 3.5 for Ph.D. admissions. We assume that the undesired prediction results from certain features outside of normal ranges and thus, we attempt to move abnormal features into normal ranges to generate CFEs with the desired prediction. Specifically, we replace an abnormal feature with the closest endpoint of its normal range. As the endpoints are stationary, CFEs after feature replacement tend to have the same feature value in different trials for the same/similar input. In practice, it may be unnecessary to move all abnormal features into normal ranges for the desired prediction. Therefore, we aim to select minimal subsets of abnormal features to replace where each subset corresponds to a CFE. CFEs determined by all minimal subsets are diverse as an arbitrary minimal subset is not contained by another subset. 

As mentioned earlier, the problem of finding CFEs boils down to selecting minimal subsets of abnormal features to replace, which can be formulated as either the Maximally Satisfiable Subsets (MSS) or Minimal Unsatisfiable Subsets (MUS) problem \cite{liffiton2013enumerating,liffiton2005finding}. However, finding all minimal sets for satisfiable CFEs is an NP-Complete problem as an exponential number of subsets should be checked. To enhance efficiency, we covert the enumeration of minimal subsets to the Boolean satisfiability problem (SAT) that finds satisfiable Boolean assignments over a series of Boolean logic formulas, which can be solved with efficient modern solvers. As for commonly mentioned constraints (e.g., actionability, correlation, and causality), we can conveniently write them in Boolean logic formulas, which can be conjugated into current clauses in conjunctive normal form (CNF). Therefore, our framework is flexible to provide feasible counterfactual recommendations. 

The main contributions of this paper are summarized as follows.
\begin{itemize}
    \item We reformulate the counterfactual explanation problem to satisfy both flexibility and robustness by replacing a minimal subset of abnormal features with the closest endpoints of normal ranges. 
    \item We convert this problem by checking the satisfiability of Boolean logic formulas for a Boolean assignment, which can be solved by modern SAT solvers efficiently. In addition, common constraints can be easily incorporated into current Boolean logic formulas and solved together. 
    \item We conduct intensive experiments on both synthetic and real-world datasets to demonstrate that our approach produces more consistent and diverse CFEs than state-of-the-art methods.
\end{itemize}

\section{Related Work}

Counterfactual explanations \cite{s2017counterfactual} refer to perturbed instances with the minimum cost that result in a different prediction from a pre-trained model given an input instance. These explanations provide ways to comprehend the model's prediction logic and offer advice to users receiving adverse predictions. Most existing algorithms focus on modeling practical requirements and user preferences with proper constraints. Typical constraints include actionability \cite{ustun2019actionable}, which freezes immutable features such as race, gender, etc; plausibility \cite{joshi2019towards,van2019interpretable}, which requires CFEs to follow the data distribution; diversity \cite{mothilal2020explaining,russell2019efficient}, which generates a diverse set of explanations at a time; sparsity \cite{dhurandhar2018explanations}, which favors fewer features changed; causality \cite{joshi2019towards,karimi2020algorithmic}, which restricts CFEs to meet specific causal relations. However, recent studies \cite{s2017counterfactual,laugel2019issues,russell2019efficient,10.1145/3531146.3533218,sokol2019counterfactual,10.1016/j.artint.2022.103840} have revealed that there often exist multiple CFEs with equivalent performance but different feature values for an input or seemingly indifferent inputs. Next, we review research that takes advantage of counterfactual multiplicity and addresses concerns regarding non-robustness. 

Multiple CFEs provide users with more flexibility to prioritize their preferences without compromising the validity and proximity of CFEs. When a single CFE is inadequate to meet users' requirements, employing a diverse set of CFEs is an effective and straightforward strategy to overcome this limitation. For example, \citeauthor{s2017counterfactual} \cite{s2017counterfactual} generate a diverse set by running multiple times with different initializations; \citeauthor{russell2019efficient} \cite{russell2019efficient} prohibits the transition to previous CFEs in each run; while \citeauthor{mothilal2020explaining} \cite{mothilal2020explaining} add a DPP (Determinantal Point Processes) term to ensure that the CFEs are far apart from each other. In addition, multiple CFEs also enable researchers to develop Human-Computer Interaction (HCI) tools for interactively satisfying user requirements \cite{9229232,10.1145/3459637.3482397}. However, such a diverse set can be inconsistent for two inputs with no or slight differences. In our paper, we aim to generate a consistent and diverse set of CFEs for an input or seemingly different inputs, to enhance the robustness and reliability of CFEs.  

The non-robustness issue of CFEs has garnered significant attention recently. As introduced in \cite{slack2021counterfactual,laugel2019issues}, even a slight perturbation to the input can result in drastically different CFEs. To verify this phenomenon for neural network models, \citeauthor{slack2021counterfactual}  \cite{slack2021counterfactual} train an adversarial model that is sensitive to trivial input changes. Some relevant works \cite{upadhyay2021towards,black2022consistent,pawelczyk2020counterfactual,dutta2022robust} propose to generate CFEs that yield consistent predictions when the model is retrained. For example, \cite{pawelczyk2020counterfactual} proves that adhering to the data manifold ensures stable predictions for CFEs; \cite{upadhyay2021towards} incorporates adversarial training to produce robust models for generating explanations; ~\citeauthor{black2022consistent}~\shortcite{black2022consistent} states that closeness to the data manifold is insufficient to indicate counterfactual stability, and they propose Stable Neighbor Search (SNS) to find an explanation with the lower model Lipschitz continuity and higher confidence. However, constraining consistent predictions of CFEs does not necessarily ensure CFEs with the same or similar feature values, and still fails to address the unfairness issue. Moreover, these robustness methods lack the flexibility to meet user requirements while our work considers both flexibility and robustness simultaneously. 

\section{Preliminary}

\subsection{Counterfactual Explanations}
Let us consider a pretrained model $f: \mathcal{X} \to \mathcal{Y}$, where $\mathcal{X} \subseteq \mathbb{R}^d $ denotes the feature space and $\mathcal{Y}$ is the prediction space. For simplicity, let $\mathcal{Y} = \{0, 1\}$, where $0$/$1$ denotes unfavorable/favorable prediction, respectively. Given an input instance $\mathbf{x} \in \mathcal{X}$, which is predicted to be the unfavorable outcome ($f(\mathbf{x})=0$), a counterfactual explanation (CFE) $\mathbf{c}$ is a data point that leads to a favorable prediction, i.e., $f(\mathbf{c})=1$, with minimal perturbations of $\mathbf{x}$. Formally, a counterfactual explanation method $g: f \times \mathcal{X} \to \mathcal{X}$ can be mathematically defined as follows:
\begin{align}
\label{eq:definition}
    \mathop{\arg\min}_{\mathbf{c}}\ &\ \text{cost}(\mathbf{x}, \mathbf{c}) \\
    s.t.\  &\ f(\mathbf{c}) = 1 \nonumber
\end{align}
where $\text{cost}(\cdot, \cdot): \mathcal{X}\times \mathcal{X} \to \mathbb{R}^{+}$ is a distance or cost metric that quantifies the efforts required in order to change from an input $\mathbf{x}$ to its CFE $\mathbf{c}$. In practice, the commonly used cost function includes $L_1$/MAD \cite{s2017counterfactual,russell2019efficient}, total log-percentile shift \cite{ustun2019actionable}, and $L_2$ norm on latent space \cite{pawelczyk2020learning}. To optimize Eqn. \eqref{eq:definition}, it can be further transformed to the Lagrangian form \cite{s2017counterfactual}, as shown below:
\begin{align}
\label{eq:lagrangian}
    \mathcal{L}(\mathbf{c}, \lambda) = \text{cost}(\mathbf{x}, \mathbf{c}) + \lambda \ell(f(\mathbf{c}), 1)
\end{align}
where $\ell(\cdot, \cdot)$ is a differential function to measure the gap between $f(\mathbf{c})$ and the favorable prediction $1$, and $\lambda$ is a positive trade-off factor. By optimizing the above objective, a CFE method $g(f, \mathbf{x})$ returns a single CFE or a set of CFEs for an input $\mathbf{x}$. 

The definition in Eqn. \eqref{eq:definition} captures the most basic form of counterfactual explanations. Usually, additional constraints are often required to ensure that the produced CFEs are useful and actionable for specific applications \cite{verma2020counterfactual}. 


\subsection{Robustness of Counterfactual Explanations}

Motivated by the formalization in \cite{artelt2021evaluating}, we formally define the robustness of CFEs in more general cases that include slight input perturbations and model changes. However,  before presenting technical details, one critical question that needs to answer is ``Do we want CFEs to remain consistent after a series of slight changes, or should they vary to reflect such changes?''.  The answer depends on practical scenarios. In certain applications, one may expect CFEs to be sensitive to such tiny changes. For example, in the study of the effects of climate change on sea turtles \cite{blechschmidt2020climate}, one may expect CFEs to be sensitive to temperature changes. In this paper, we assume that such trivial changes are either irrelevant or less important to the generation of CFEs. As such, we aim to produce CFEs that are robust to trivial changes, such as inputs added with random noise, and model retraining on new data from the same distribution. 

Let $\hat{\mathbf{x}}$ represent a slightly perturbed sample that is close to $\mathbf{x}$, meaning $\hat{\mathbf{x}} \sim p(\mathbf{x})$, where $p(\mathbf{x})$ is the density estimation of perturbed samples that yield the same prediction as the input $\mathbf{x}$. Similarly, let $\hat{f} \in \mathcal{F}$ denote a retrained model belonging to the class $\mathcal{F}$, which consists of potential models that perform equivalently well as the original one.

\begin{definition}[Robustness of Counterfactual Explanations]
Given a function $d(\cdot, \cdot)$ computing the similarity between two sets of CFEs, we quantify the robustness of the explanations $g(f,\mathbf{x})$ by assessing the expected similarity between the current set of CFEs, and a new set of CFEs after potential input perturbations or model changes.
\begin{align}
    \mathop{\mathbb{E}}_{\mathbf{x}' \sim p(\mathbf{x}), \atop
\hat{f}\in \mathcal{F}}
    [d(g(f, \mathbf{x}), g(\hat{f}, \mathbf{x}'))]
\end{align}
\end{definition}

A lower value indicates higher robustness. By minimizing the above expectation, we can generate robust CFEs. However, in real life, $p(\mathbf{x})$ and $\mathcal{F}$ are typically unknown. Intuitively, they can be determined based on specific changes that users desire to be robust against. For instance, one can consider adding Gaussian noise to the input, masking certain features, or retraining the models on data from the same distribution, to decide $p(\mathbf{x})$ and $\mathcal{F}$.

\begin{figure}[tp]
    \centering
    \includegraphics[width=\linewidth]{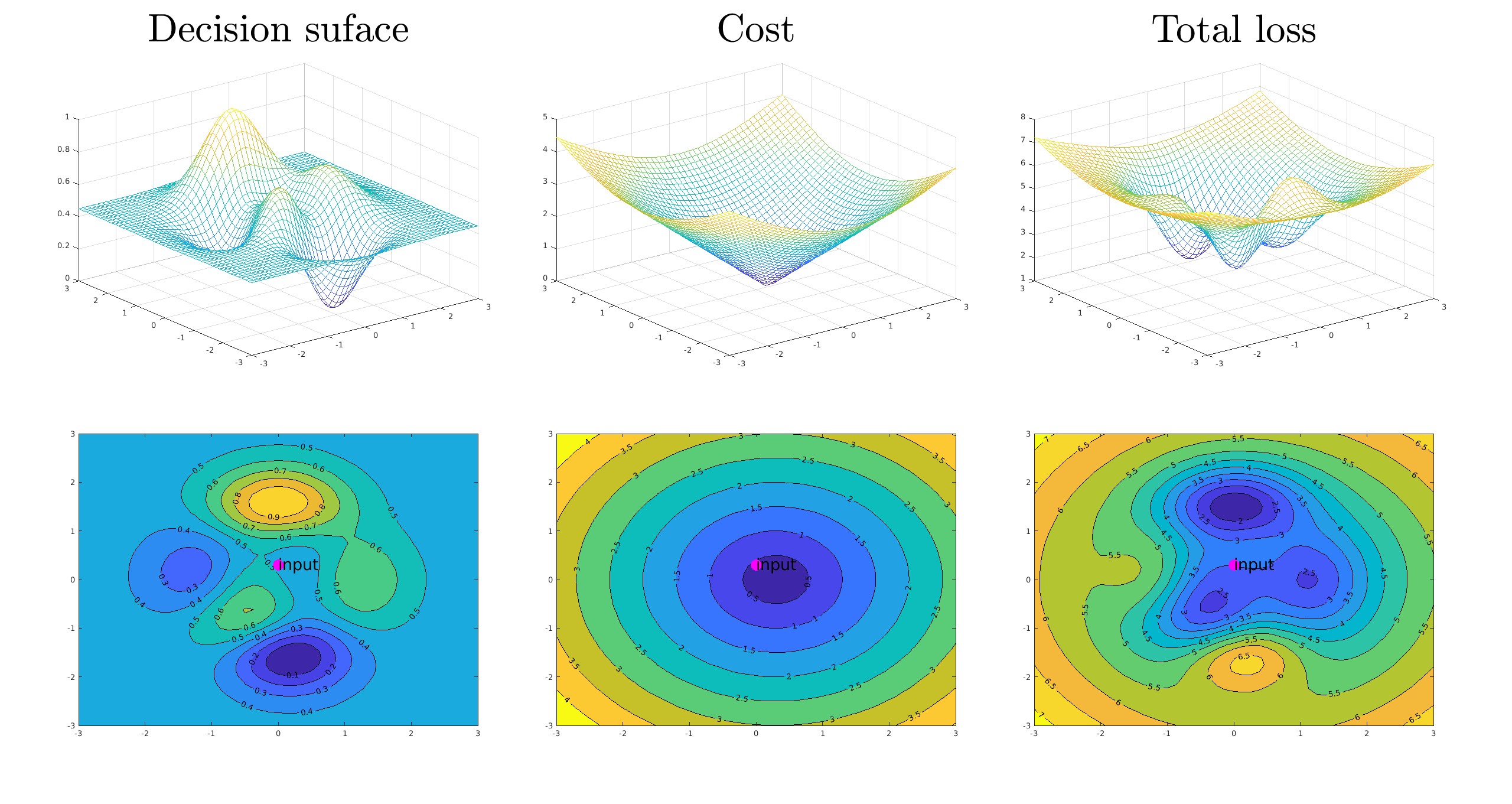}
    \caption{The toy example to illustrate the non-robustness issue caused by input perturbations for an input $(0.3,0)$. The figures in the top row depict the decision surface, $L_2$ cost, and total loss in the Lagrangian form ($\lambda=5$ for simplicity). The bottom row shows the contour lines corresponding to each figure above. It is evident that even a slight perturbation on this input can result in different local minima.}
    \label{fig:toy_example}
\end{figure}

\subsection{Causes of Non-robustness}

Here, we explain the root causes of non-robustness. The total loss in Eqn. \eqref{eq:lagrangian} form is usually non-convex due to the non-convex decision surface of probabilistic models and other constraints. As shown in Figure \ref{fig:toy_example}, multiple local minima can be found, but current methods often select a single or $k$ CFEs from them. Such selected CFEs can be different in each trial. Next, we discuss several influential factors that result in non-robust CFEs. 

\begin{itemize}
    \item Input perturbations. Input instances can be perturbed by adding some noise or masking random features. Due to the local sensitivity of large models, such trivial perturbations can significantly influence model predictions, leading to different counterfactual explanations (CFEs) \cite{slack2021counterfactual}. 

    \item Model updates. The predictive model $f$ in Eqn.~\eqref{eq:definition} is typically retrained periodically in deployment. The updated model $f'$ may exhibit slightly different behavior compared to the previous model and thus may have a great impact on the cost of the desired prediction. 
    
    \item Random factors. Heuristic search methods $g(f, \mathbf{x})$ for Eq.~\eqref{eq:lagrangian} often involve random factors, e.g., random initial points in gradient descent \cite{s2017counterfactual}, random samples in Growing Sphere \cite{lash2017generalized,pawelczyk2020learning}, and random selection in genetic algorithms \cite{10.1145/3375627.3375812}. Algorithm randomness often causes different solutions in non-convex situations. 
    \item Algorithmic configurations. Different configurations of $g(f, \mathbf{x})$ also affect the search process, e.g., the trade-off factor $\lambda$, the step size in gradient descent, the stop condition, the guide point \cite{van2019interpretable}, the radius of the sphere \cite{lash2017generalized,pawelczyk2020learning}, variational autoencoders used to approximate data distribution \cite{pawelczyk2020learning,joshi2019towards}. 
\end{itemize}

\section{Proposed Method}

In this section, we propose a method offering a diverse set of CFEs that remain stable when the input or model being explained is changed subtly by incorporating domain knowledge into the search.

\subsection{Problem Formulation}
\label{sec:41}

Let $\mathcal{N} = \{1, 2, \ldots, d\}$ be the indices of $d$-dimensional input features. $\mathbf{x}^{i}, i \in \mathcal{N}$ denotes the $i$-th feature value of an input instance $\mathbf{x}$, and $[\mathbf{a}^i, \mathbf{b}^i]$ denotes the normal range of $i$-th feature. We assume the normal range of each feature can be acquired from domain knowledge. For example, the normal range of BMI \cite{CDC_USA}  is $[18.5, 24.9]$, and the normal heart rate \cite{shmerling-2020} is between 60 to 100. For a given input $\mathbf{x}$, then we partition $\mathcal{N}$ into two disjoint sets by checking whether each feature is in the normal range or not:  $\mathcal{N}_0$ (with size $d_0$) for abnormal features and $\mathcal{N}_1$ (with size $d_1$) for features within the normal range.

We further assume that undesired predictions are attributed to abnormal features. Our intuition is to bring an abnormal feature within its normal range, which would increase the probability of the desired prediction. To ensure consistency, we replace an abnormal feature with the closest endpoint of the corresponding normal range. For instance, if the BMI of an obese person is $40$, and the suggested BMI is $24.9$ (BMI normal range $[18.5, 24.9]$). We aim to convert a subset of abnormal features into their respective normal ranges, to achieve lower cost and higher sparsity. Let $\mathcal{A} \subseteq \mathcal{N}_0$ represent the subset of abnormal features to be replaced. We employ the function $\eta(\cdot,\cdot) \in \mathbb{R}^d$ to calculate the values after replacement for the abnormal features, where the $i$-th element is computed as follows:

\begin{align}
    \eta(\mathbf{x}, \mathcal{A})^i = \begin{cases}
\mathbf{a}^i, & i \in \mathcal{A} \text{ and } \mathbf{x}^i < \mathbf{a}^i\\
\mathbf{b}^i, & i \in \mathcal{A} \text{ and } \mathbf{x}^i > \mathbf{b}^i\\
\mathbf{x}^i, & \text{otherwise}
\end{cases}
\end{align}

In applications where normal ranges are not available (we set $\mathcal{N}_0 = \mathcal{N}$), we replace the input feature values with the corresponding values of a guide point or prototype that has the desired prediction. Correspondingly, given a guide point or prototype $\hat{\mathbf{x}}$,  $\eta(\cdot,\cdot)^i$ can be expressed as follows:
\begin{align}
    \eta(\mathbf{x}, \mathcal{A})^i = \begin{cases}
\hat{\mathbf{x}}^i, & i \in \mathcal{A}\\
\mathbf{x}^i, & \text{otherwise}
\end{cases}
\end{align}

We use the binary vector $\mathbf{m}(\mathcal{A}) \in \{0,1\}^d$ (abbreviated as $\mathbf{m}$) to represent the presence or absence of each feature in subset $\mathcal{A}$. Specifically, each element $\mathbf{m}^i$ indicates whether the $i$-th feature is included in $\mathcal{A}$, i.e., $\mathbf{m}^i = \mathds{1}_{i \in \mathcal{A}}$.
The feature replacement operation for an input $\mathbf{x}$ on subset $\mathcal{A}$ is defined as follows:
\begin{align}
r(\mathbf{x}, \mathcal{A}) = \mathbf{x} \odot (1- \mathbf{m}) + \eta(\mathbf{x}, \mathcal{A}) \odot \mathbf{m}
\end{align}
where $\odot$ is the Hadamard (element-wise) product of vectors. If the $i$-th feature is not included in $\mathcal{A}$, we keep its original value, otherwise, we change it to the closest endpoint of its normal range. 

As there could be a large number of subsets $\mathcal{A}$ that can achieve the desired prediction after feature replacement, we aim to find the minimal subsets $\mathcal{A}^*$. Note that a minimal subset does not necessarily indicate minimal cardinality due to the difficulty of comparing the costs of different features. For example, increases in ``academic degree'', ``credit score'', and ``salary'' cannot be measured concisely in a mathematical cost. As such, we treat the subset (on ``academic degree'', ``credit score'') and the subset (on ``salary'') as two different solutions, although the latter subset has the smaller cardinality, same to \cite{10.1145/3459637.3482397}. Based on the above discussion, we formulate this problem to generate nondominated sets of CFEs as follows:

\begin{definition}[Minimal Satisfiable Counterfactual Explanations]\label{def:cf_normal}
Given a pre-trained model $f$, input $\mathbf{x}$ and feature normal ranges, the goal is to find all minimal satisfiable subsets, where for each subset $\mathcal{A}^*\subseteq\mathcal{N}_0$, $f(r(\mathbf{x}, \mathcal{A}^*))$ can achieve the desired prediction and $r(\mathbf{x}, \mathcal{A}^*)$ represents a counterfactual explanation $\mathbf{c}$.
\end{definition}

\textbf{Diversity Analysis:} In \cite{mothilal2020explaining}, authors propose a diversity metric over a set of CFEs of size $k$, named \textit{count-diversity}, to measure the inner discrepancy, written by,
\begin{equation}
\label{eq:cdiversity}
    \textit{count-diversity} =  \frac{2}{k(k-1)d} \sum_{i=1}^{k-1} \sum_{j=i+1}^k \sum_{l=1}^d \mathds{1}_{[\mathbf{c}_i^l \neq \mathbf{c}_j^l]}
\end{equation}
\begin{theorem}The lower bound of \textit{count-diversity} defined in Eqn. \eqref{eq:cdiversity} over solutions of our problem definition is $\frac{2}{d}$.
\end{theorem}
\begin{proof} 
Let's denote the inner loop term $\sum_{l=1}^d \mathds{1}_{[\mathbf{c}_i^l \neq \mathbf{c}_j^l]}$ as $D(\mathbf{c}_i, \mathbf{c}_j)$ for brevity. We aim to prove that for any two arbitrary CFEs $\mathbf{c}_i$ and $\mathbf{c}_j$, the pairwise distance $D(\mathbf{c}_i, \mathbf{c}_j)$ is always greater than or equal to $2$. To demonstrate this, we consider two contradictive cases by assuming $0 \le D(\mathbf{c}_i, \mathbf{c}_j) < 2$.

\begin{case}[$D(\mathbf{c}_i, \mathbf{c}_j)=0$]
This case implies $\mathbf{c}_i=\mathbf{c}_j$, which contradicts the fact that two distinct CFEs are from the solution set. 
\end{case}
\begin{case}[$D(\mathbf{c}_i, \mathbf{c}_j)=1$]
In this case, there is only one feature difference. Let $\mathcal{A}_i$ and $\mathcal{A}_j$ denote the indices of abnormal features of two solutions. Then, $\mathcal{A}_i \subseteq \mathcal{A}_j$ or $\mathcal{A}_j \subseteq \mathcal{A}_i$ must exist, which also contradicts the minimal set to return in our problem definition. One CFE should be excluded because it costs more than the other. 
\end{case}
The above contradictive proof shows that $D(\mathbf{c}_i, \mathbf{c}_j) \ge 2$ holds. Summing up all $\frac{k(k-1)}{2}$ pair-wise distance $D(\mathbf{c}_i, \mathbf{c}_j)$, we can obtain the lower bound $\frac{2}{d}$.
\end{proof}

\textbf{Robustness Analysis:} Let $\mathbf{z}=\mathbf{c}-\mathbf{x}$ represent the recommended actions for a user. In our method, $\mathbf{z}$ consistently applies to slightly perturbed instance $\hat{\mathbf{x}}$, except in the following two situations: (1) $f(\hat{\mathbf{x}}+\mathbf{z})$ is no longer valid, which occurs when slight perturbations have a negative impact on the desired prediction. For example, normal features may be turned into abnormal ones. We need more effort than $\mathbf{z}$ to achieve the desired prediction. (2) Changing fewer abnormal features is sufficient to achieve the desired prediction, indicating that slight perturbations are beneficial. In this case, $\mathbf{z}$ is omitted as there exist more cost-efficient solutions. As both model continuity and perturbation strategies can influence the $\mathbf{z}$, we leave the determination of the maximal bound of perturbation, to which our method remains robust, for future work. 

\subsection{Problem Solving}
The brute-force method that evaluates all possible subsets is exponentially complex with respect to the number of abnormal features. Next, we propose a technique \fullname to boost the search process. Our method starts with finding the binary vectors $\mathbf{m}$ that satisfy the desired prediction after feature replacement. This can be converted into the Boolean satisfiability problem that checks whether there exists a Boolean value assignment on $d_0$ variables (features in abnormal ranges) such that the conjunction of Boolean formulas evaluates to $True$. For better efficiency, we introduce the following proposition from domain knowledge.
\begin{figure}
    \centering
    \includegraphics[width=\linewidth]{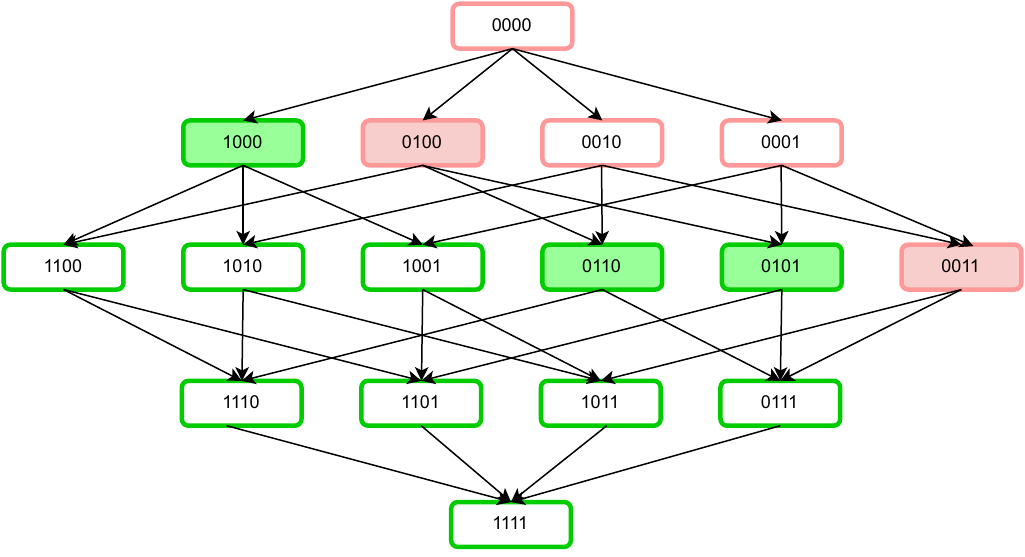}
    \caption{The figure shows all subsets of a toy example with $4$ abnormal features. The bitvectors denote the binary vector $\mathbf{m}$. Boxes with red/green borders represent the unsatisfiable/satisfiable subsets respectively. The minimal subsets for CFEs are filled with green background and the maximal unsatisfiable subsets are filled with red background.}
    \label{fig:example}
\end{figure}
\begin{prop}[Monotoncity of $f(r(\mathbf{x}, \cdot))$]
The function $f(r(\mathbf{x}, \cdot))$ is monotone, that is, $f(r(\mathbf{x}, \mathcal{A})) \le f(r(\mathbf{x}, \mathcal{B}))$ holds for all $\mathcal{A} \subseteq \mathcal{B} \subseteq \mathcal{N}_0$. 
\end{prop}

This proposition aligns with common sense in practical applications. It is important to note that the undesired prediction arises from specific abnormal features according to our assumption. Intuitively, moving an abnormal feature into the normal range should never decrease the desired probability. Additionally, we assume that the predictive model $f$ has learned the relationship between feature normal ranges and the predicted classes.

Based on the monotonicity of function $f(r(\mathbf{x}, \cdot))$, we derive the following two theorems for any $\mathcal{A} \subseteq \mathcal{B} \subseteq \mathcal{N}_0$.

\begin{theorem}\label{theorem:superset}
If $f(r(\mathbf{x}, \mathcal{A}))$ can achieve the desired target, $f(r(\mathbf{x}, \mathcal{B}))$ can also achieve the desired target for any superset $\mathcal{B}$ of $\mathcal{A}$. 
\end{theorem}
\begin{theorem}\label{theorem:subset}
If $f(r(\mathbf{x}, \mathcal{B}))$ cannot achieve the desired target, $f(r(\mathbf{x}, \mathcal{A}))$ cannot satisfy the desired target either for any subset $\mathcal{A}$ of $\mathcal{B}$. 
\end{theorem}
\begin{proof}
We first prove theorem \ref{theorem:superset}. If $f(r(\mathbf{x}, \mathcal{A})) \ge \delta$, we can induce that $f(r(\mathbf{x}, \mathcal{B})) \ge \delta$ as $f(r(\mathbf{x}, \mathcal{B})) \ge f(r(\mathbf{x}, \mathcal{A}))$ holds for $\mathcal{A} \subseteq \mathcal{B}$, where $\delta$ is the confidence threshold of desired prediction. Similarly, theorem \ref{theorem:subset} can be proved.    
\end{proof}
Theorem \ref{theorem:superset} illustrates that if we can achieve the desired prediction by replacing abnormal features in $\mathcal{A}$, there is no need to change more abnormal features. The theorem tends to produce sparser results at a lower cost. Theorem \ref{theorem:subset} demonstrates that if we cannot achieve the desired prediction by changing abnormal features in $\mathcal{B}$, there is no need to check the satisfiability of any subsets of $\mathcal{B}$. 

To prune as many as subsets at a time with these two theorems, we need to find the minimal satisfiable subset (MSS) and the maximal unsatisfiable subset (MUS), shown as boxes filled with green/red background in Figure \ref{fig:example}. 
Next, we introduce two algorithms to achieve this: Grow($\cdot$) in Algorithm \ref{alg:grow} and Shrink($\cdot$) in Algorithm \ref{alg:shrink}. The Grow($\cdot$) algorithm starts with an arbitrary unsatisfiable subset $\hat{\mathcal{A}}$ and iteratively attempts to change other abnormal features until a maximal unsatisfiable subset is found. The Shrink($\cdot$) algorithm starts with an arbitrary satisfiable subset and iteratively attempts to reserve some features until a minimal satisfiable subset is found. Note that Grow($\cdot$) and Shrink($\cdot$) algorithms serve two plugins in our method, which can be replaced by any advanced algorithm with the same purpose. 

\begin{algorithm}[!t]
\caption{Grow($\mathcal{A}$)}\label{alg:grow}
\begin{algorithmic}[1]
\Require An unsatisfiable subset $\mathcal{A}$.
\Ensure A maximal unsatisfiable subset $\hat{\mathcal{A}}$.
\State $\hat{\mathcal{A}} = \mathcal{A}$
\For {$i\in \mathcal{N}_0 \backslash \mathcal{A}$}
    \If {$f(r(\mathbf{x}, \hat{\mathcal{A}} \bigcup \{i\})) == 0$}
        \State $\hat{\mathcal{A}} = \hat{\mathcal{A}} \bigcup \{i\}$
    \EndIf
\EndFor
\State \Return $\hat{\mathcal{A}}$
\end{algorithmic}
\end{algorithm}

\begin{algorithm}[!t]
\caption{Shrink($\mathcal{A}$)}\label{alg:shrink}
\begin{algorithmic}[1]
\Require A subset $\mathcal{A}$ satisfying the desired prediction.
\Ensure A minimal subset $ \mathcal{A}^*$ for a CFE. 
\State $\mathcal{A}^* = \mathcal{A}$
\For {$i\in \mathcal{A}$}
    \If {$f(r(\mathbf{x}, \mathcal{A}^* \backslash \{i\})) == 1$}
        \State $\mathcal{A}^* = \mathcal{A}^* \backslash \{i\}$
    \EndIf
\EndFor
\State \Return $\mathcal{A}^*$
\end{algorithmic}
\end{algorithm}

Further, we introduce how to solve it under Boolean satisfiability problem \cite{liffiton2013enumerating,bendik2018recursive}. In particular, any subset can be converted to a satisfiable Boolean assignment under a set of propositional logic formulas in conjunctive normal form (CNF), i.e., $\mathcal{A} \iff \mathbf{m}: \text{CNF}=True$. For example, for a subset $\mathcal{A} = \{1, 2\}$ in Figure \ref{fig:example}, we can write the $\text{CNF}$ in the following equation, and $[1,1,0,0]$ is the only solution. 
\begin{equation}
    \text{CNF} = \mathbf{m}^1 \land \mathbf{m}^2 \land \lnot \mathbf{m}^3 \land \lnot \mathbf{m}^4
\end{equation}

By employing this approach, the explicit materialization of all subsets can be avoided, thereby mitigating the exponential space complexity. The crux lies in devising the appropriate propositional logic formulas.

Our complete algorithm is shown in Algorithm \ref{alg:mscf}. 
Initially, in line $1$, we merely forbid the changes on normal features and any possible binary assignments on abnormal features can satisfy the $\text{CNF}$. $\text{getMask}(\text{CNF})$ uses an SAT solver to return a solution satisfying the CNF in line $3$. In our paper, we adopt the Z3 package\footnote{https://github.com/Z3Prover/z3}.   In line $7$, we convert the binary vector $\mathbf{m}$ to indices of subsets. Next, we check whether replacing features in this subset can achieve the desired prediction. If not satisfying the desired prediction, we call the Grow($\cdot$) algorithm to find the maximal unsatisfiable subset and then prune all subsets of it. The prune operation is achieved by the following propositional Boolean formulas which are conjugated into the existing CNF.
\begin{equation}
    \text{PruneSubSet}(\hat{\mathcal{A}}) = \vee _{i: i \in \mathcal{N}_0 \backslash \hat{\mathcal{A}}}\  \mathbf{m}^i
\end{equation}
Similarly, if we find a subset satisfying the desired prediction, we call Shrink($\cdot$) function to return a minimal satisfiable subset, which can induce a CFE with minimal perturbations of features. Then, we prune all supersets of it with the following logic formula,
\begin{equation}
    \text{PruneSuperSet}(\mathcal{A}^*) = \vee _{i: i \in \mathcal{A}^*}\ \lnot \mathbf{m}^i
\end{equation}
If no solution satisfies the CNF in the current iteration, we stop our algorithm in the line $4$ and $5$. 

\textbf{Correctness.} In our algorithm, each subset is either evaluated to be an MSS/MUS or pruned by an MSS/MUS. Hence, our algorithm will return all minimal CFEs, providing the same solutions as the brute-force search. 

\textbf{Space Complexity.} The space complexity depends on how many minimal CFEs are returned. The worst case is $O(\binom{d_0}{\lfloor d_0/2 \rfloor})$, which corresponds to that feature replacements on arbitrary abnormal features of size $\lfloor d_0/2 \rfloor$ are satisfiable. 

\textbf{Time Complexity.} The runtime of our method primarily depends on two parts: (1) a solver (line 3) that takes a set of constraints and returns a mask $\mathbf{m}$. The solver is considerably faster than calling the pretrained model. (2) evaluating the prediction on a subset.  Compared with brute-force search, which calls the deep model $2^{d_0}$ times to check all $2^{d_0}$ possible sets, our method reduces the number of calls on the model by pruning on certain subsets and supersets. Therefore, the empirical running time will decrease. 

\begin{algorithm}[t]
\caption{CEMSP($\mathbf{x}, f$)}\label{alg:mscf}
\begin{algorithmic}[1]
\Require An input $\mathbf{x}$, a pretrained model $f$.
\Ensure All minimal subsets $\mathcal{A}^*$ for CFEs.
\State $\text{CNF} = \land _{i: i \in \mathcal{N}_1}\ \lnot \mathbf{m}^i$
\While{True}
\State $\mathbf{m} \leftarrow \text{getMask}(\text{CNF})$
\If {not $\mathbf{m}$} \Comment{No assignment $\mathbf{m}$ returned.}
\State \textbf{Break}
\EndIf
\State {$\mathcal{A} \leftarrow \{i \in \mathcal{N}_0: \mathbf{m}^i = 1\}$ }
\If{$f(r(\mathbf{x}, \mathcal{A})) == 0$}
\State $\hat{\mathcal{A}} \leftarrow \text{Grow}(\mathcal{A})$
\State $\text{CNF} \to \text{CNF} \wedge \text{PruneSubset}(\hat{\mathcal{A}})$ \Comment{Prune any subset of $\hat{\mathcal{A}}$}
\Else
\State $\mathcal{A}^* \leftarrow \text{Shrink}(\mathcal{A})$
\State $\text{\textbf{yield}  }\mathcal{A}^*$
\State $\text{CNF} \to \text{CNF} \wedge \text{PruneSuperset}(\mathcal{A}^*)$ \Comment{Prune any superset of $\mathcal{A}^* $}
\EndIf
\EndWhile
\end{algorithmic}
\end{algorithm}

\subsection{Compatibility with Other Constraints}

The major theme of recent research is to model various constraints into CFE generation. Here, we show how to write these constraints by propositional logic formulas that can be conjugated into the CNF in line $1$ of our Algorithm \ref{alg:mscf}. 

\textbf{Immutable features.} Considering some features are immutable (e.g., race, birthplace), CFEs should avoid perturbations on these features. To achieve this, we add the following Boolean logic formula for a set of immutable features $\mathcal{I}$,  
\begin{align}
    \text{Actionability}(\mathcal{I}) = \land_{i\in \mathcal{I}}\ \lnot \mathbf{m}^i
\end{align}
Accordingly, these features should be ignored in Grow($\cdot$) algorithm when it searches for the maximal unsatisfiable set. Alternatively, we can directly treat immutable features as normal features to avoid any changes to them. 

\textbf{Conditional immutable features.} These features must change in one direction, e.g., education degree.  We can examine whether moving a feature value into its normal range follows the valid direction. If violating the valid direction, we treat this feature as an immutable feature, otherwise, we put no restriction on this feature. 

\textbf{Causality.} In practice, changing one feature may cause a change in other features. Such causal relations among features are generally written by a set of triplets in the structural causal model (SCM) \cite{pearl2009causality}, that is, $\mathcal{M} = \langle U, V, F \rangle$, where $U$ are exogenous features, $V$ are endogenous features, and $F:U \to V$ is a set of functions that describe how endogenous features are quantitatively affected by exogenous features. To adapt causality to our method, we only keep these triplets that normal exogenous features lead to normal endogenous features as our method merely considers discrete feature changes (from abnormal to normal). For example, feature $\mathbf{x}^1$ is an exogenous feature that affects two endogenous features $\mathbf{x}^2$ and $\mathbf{x}^3$, and $\mathbf{x}^2$ and $\mathbf{x}^3$ become normal in the consequence of its normal ancestor feature $\mathbf{x}^1$.  In this example, we can add two material conditions in the following that restrict the feature change of CFEs to follow the causal relations. 
\begin{align}
    \text{Causality} = (\lnot \mathbf{m}^1 \vee \mathbf{m}^2) \land (\lnot \mathbf{m}^1 \vee \mathbf{m}^3)
\end{align}
At the same time, Grow($\cdot$) and Shrink($\cdot$) algorithms should be updated to satisfy such causal relations when they attempt to add/remove a feature. This can be easily implemented by storing these causal relations by an inverted index where an entry is an exogenous feature and the inverted list contains all its endogenous features. 

\textbf{Correlation.} Correlation can be regarded as bidirectional causal relations. For example, if features $\mathbf{x}^1$ and $\mathbf{x}^2$ are correlated and in the normal range simultaneously, we can write the correlation between $\mathbf{x}^1$ and $\mathbf{x}^2$ as, 
\begin{align}
    \text{Correlation} = (\lnot \mathbf{m}^1 \vee \mathbf{m}^2) \land (\lnot \mathbf{m}^2 \vee \mathbf{m}^1) 
\end{align}

The great advantage of our framework is that it allows us to insert these constraints gradually and flexibly, as the complete relation graphs (e.g., full causal graph) are often difficult to derive in the beginning.

\section{Experiments}
In this section, we undertake a quantitative comparison between our proposed method \shortname\ and state-of-the-art approaches. Additionally, we demonstrate empirical examples of counterfactual explanations that effectively integrate practical constraints. The source code is available at the GitHub repository\footnote{https://github.com/wangyongjie-ntu/CEMSP}. 

\paragraph{Datasets.}
We conducted a comprehensive series of experiments involving a synthetic dataset and two real-world UCI medical datasets. Notably, the medical datasets encompass diagnostic features with well-defined and clinically significant normal ranges.

\begin{itemize}
    \item \textbf{Synthetic Dataset} is a binary class dataset consisting of $20,000$ samples with $4$ features. Each feature is sampled from the normal distribution independently. Regarding label balance, the binary label $y$ is assigned a value of $1$ when the following equation is satisfied; otherwise, $y$ is set to 0:
    \begin{align*}
        (\mathbf{x}^1 > 0.5) \lor (\mathbf{x}^2 > 0.4 \land \mathbf{x}^3 > 0) \lor (\mathbf{x}^2 > 0.4 \land \mathbf{x}^3 > 0.5)
    \end{align*}
    We set the lower values of normal ranges of four features as $[0.55, 0.45, 0.05, 0.55]$ for a higher confidence prediction.
    \item \textbf{UCI HCV Dataset} \cite{diaconis1983computer}. This dataset contains $615$ instances. Following \cite{bhatt-2021}, we convert $5$ categories of diagnosis into binary classes. After label conversion, the dataset consists of $75$ individuals diagnosed with HCV and $540$ individuals labeled as healthy.  Next, we remove the ``Age'' and ``Sex'' and keep the other $10$ medical features with normal ranges.  We adopt the tight normal ranges from laboratory tests in \cite{hoffmann2016simple} as certain normal ranges depend on ``Sex'' and we remove the ``Sex'' attribute in preprocessing.
    \item \textbf{UCI Thyroid Dataset} \cite{10.5555/33429.33438}. The raw dataset contains $3,772$ instances where each instance is described by $15$ features and labeled as either hypothyroid or normal class. We retain the most discriminative features ``FTI'', ``TSH'', ``T3'', ``TT4'' that have meaningful normal ranges and remove other features. Subsequently, we drop certain rows with missing values. The final dataset consists of $223$ patients and $2530$ healthy users. Normal ranges of ``TSH'', ``T3'', ``TT4'' are from laboratory tests in \cite{joshi2011laboratory}. As we do not find the normal range of ``FTI'' that matches the ``FTI'' values in this dataset, we simply choose the 1-sigma interval of ``FTI'' of the normal group.
   
\end{itemize}

\paragraph{Evaluation Metrics. }
To comprehensively compare over CFEs across various approaches, we employ the following evaluation metrics.

\begin{itemize}
    \item \textbf{Inconsistency}. We propose to adopt a modified Hausdorff distance \cite{dubuisson1994modified,jesorsky2001robust} to measure the inconsistency between two sets of CFEs $\mathcal{C}$ and $\mathcal{C}'$,  
    \begin{align}
        H(\mathcal{C}, \mathcal{C}') = \max(h_{mod}(\mathcal{C}, \mathcal{C}'), h_{mod}(\mathcal{C}', \mathcal{C}))
    \end{align}
    where $h_{mod}(\mathcal{C}, \mathcal{C}') = \frac{1}{|\mathcal{C}|} \sum_{\mathbf{c} \in \mathcal{C}}  \min_{\mathbf{c}'\in \mathcal{C}'} ||\mathbf{c} - \mathbf{c}'||_2$ and the lower $H$ is better. 
    \item \textbf{Average Percentile Shift (APS)} \cite{pawelczyk2020learning} measures the relative cost of perturbations of CFEs, 
    \begin{align}
        \text{APS}(\mathbf{x}, \mathcal{C}) = \frac{1}{d*|\mathcal{C}|} \sum_{\mathbf{c} \in \mathcal{C}} \sum_{i=1}^d |Q^i(\mathbf{c}^i) - Q^i(\mathbf{x}^i)|
    \end{align}
    where $Q^i(\cdot)$ denotes the percentile of the $i$-th feature value relative to all values of the feature in the whole data set. A lower score is favored. 
    \item \textbf{Sparsity}. It measures the percentage of features that remain unchanged and we prefer higher sparsity,
    \begin{align}
        \text{Sparsity}(\mathbf{x}, \mathcal{C})= \frac{1}{d*|\mathcal{C}|}\sum_{\mathbf{c} \in \mathcal{C}} \sum_{i=1}^d \mathds{1}_{\mathbf{c}^i =\mathbf{x}^i}. 
    \end{align}
    \item \textbf{Diversity}. We consider two diversity metrics, named Diversity, which is introduced in \cite{mothilal2020explaining}, and count-diversity (named C-Diversity for abbreviation), which is defined in Eqn ~\eqref{eq:cdiversity} in Section \ref{sec:41}, to measure the discrepancy within returned solutions. 
    \begin{align}
        \text{Diversity} = \frac{2}{k(k-1)} \sum_{i=1}^{k-1} \sum_{j=i+1}^k \text{dist}(\mathbf{c}_i, \mathbf{c}_j)
    \end{align}
    where $\text{dist}(\cdot, \cdot)$ represents the $L_1/MAD$ and $k$ is the number of CFEs.  
\end{itemize}

\paragraph{Baselines.}
We compare our method with the following baseline methods. 
\begin{itemize}
    \item \textbf{GrowingSphere (GS)} \cite{lash2017generalized}. This algorithm searches for CFEs from random samples in the sphere neighborhood of the input. The radius of the sphere grows until a CFE is found. It adopts the postprocessing on returned CFEs to make sparser solutions. 
    \item \textbf{PlainCF} \cite{s2017counterfactual}. It minimizes the objective in Eqn. \eqref{eq:lagrangian} with gradient descent. We run this algorithm from a random initial point and stop when an iteration threshold is reached or the loss difference is below a specified threshold.
    \item \textbf{CFProto} \cite{van2019interpretable}. It adds a prototype term to restrict that CFEs should resemble the prototype of the desired class. In our experiment, we set the prototype as the closest endpoints of normal ranges of these abnormal features, that is, $r(\mathbf{x}, \mathcal{N}_0)$. 
    \item  \textbf{DiCE} \cite{mothilal2020explaining}. Compared with PlainCF, it considers the diversity constraint that is modeled by a $dpp(\cdot)$ term over a set of CFEs.
    \item  \textbf{SNS} \cite{black2022consistent}. It finds a CFE with higher confidence and lower Lipschitz constant in the neighborhood of a given CFE, to produce consistent prediction under model update.
\end{itemize}
 
\paragraph{Experiment configurations.}
We first randomly split the datasets into train/test sets at the ratio of $7:3$ and normalize all features by a standard scaler on two UCI datasets (no feature normalization on the synthetic dataset). Then, we train a 3-layer Multilayer perceptron (MLP) model $f$ with Adam optimizer. The test accuracies are 99\%, 96\%, and 98\% on three datasets correspondingly. As we intend to convert unhealthy patients to healthy ones, we produce CFEs for all correctly classified patients in test sets for two UCI datasets. For saving time, we only produce CFEs for 100 random true negative samples in the synthetic dataset. 

Our method produces a set of CFEs of varied size $k$, while GS, PlainCF, CFProto, and SNS generate one at a time. We run GS, PlainCF, and CFProto $k$ times for fair comparisons to generate the same number of CFEs as ours. For DiCE, we directly keep the size of the diverse set as ours. We evaluate all CFE methods $g(f, \mathbf{x})$ under the following two kinds of slight updates to measure the algorithm robustness. 

\begin{itemize}
    \item \label{setting:1} Inputs are fixed, and we produce two sets of CFEs from two models that are trained on the same dataset with different initializations. 
    \item \label{setting:2} Model is fixed, and we produce two sets of CFEs from an input $\mathbf{x}$ and its perturbed instance $\mathbf{x}'$, where $\mathbf{x}'=\mathbf{x}+ \alpha$ and $\alpha$ is the random noise sampling from a Gaussian distribution  $\mathcal{N}(0, \sigma)$. In our experiments, $\sigma \in \{0.0001, 0.001, 0.01, 0.1\}$.  
\end{itemize}

\begin{figure}[!tp]
    \centering
    \includegraphics[width=0.95\linewidth]{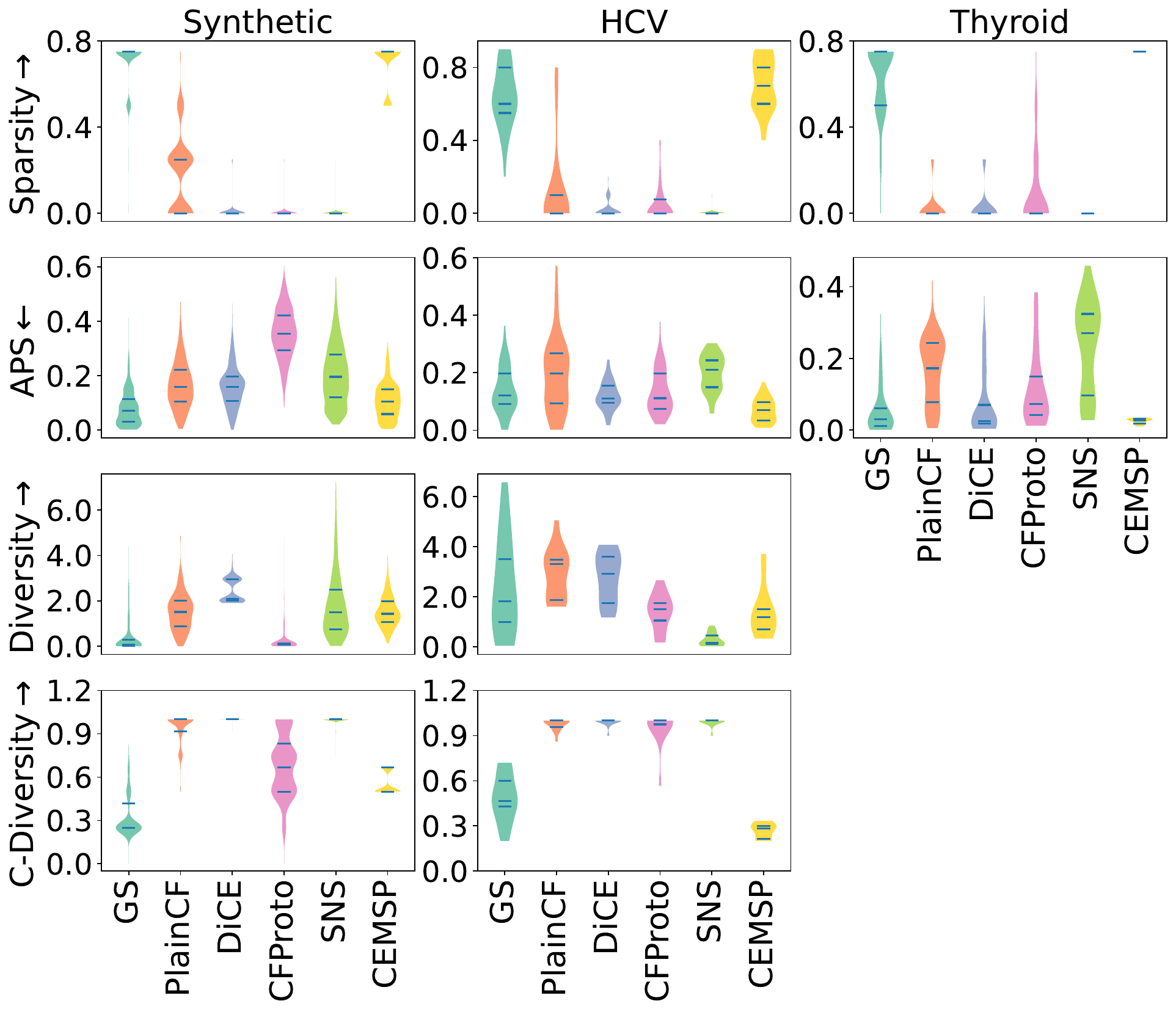}
    \caption{Evaluation of sparsity, APS, Diversity, and C-Diversity over three datasets. The $\uparrow$/$\downarrow$ means the higher/lower score is better. Diversity and C-Diversity of the Thyroid dataset are missing as our method \shortname\ only produces a single counterfactual explanation.}
    \label{fig:expscore}

    \includegraphics[width=0.95\linewidth]{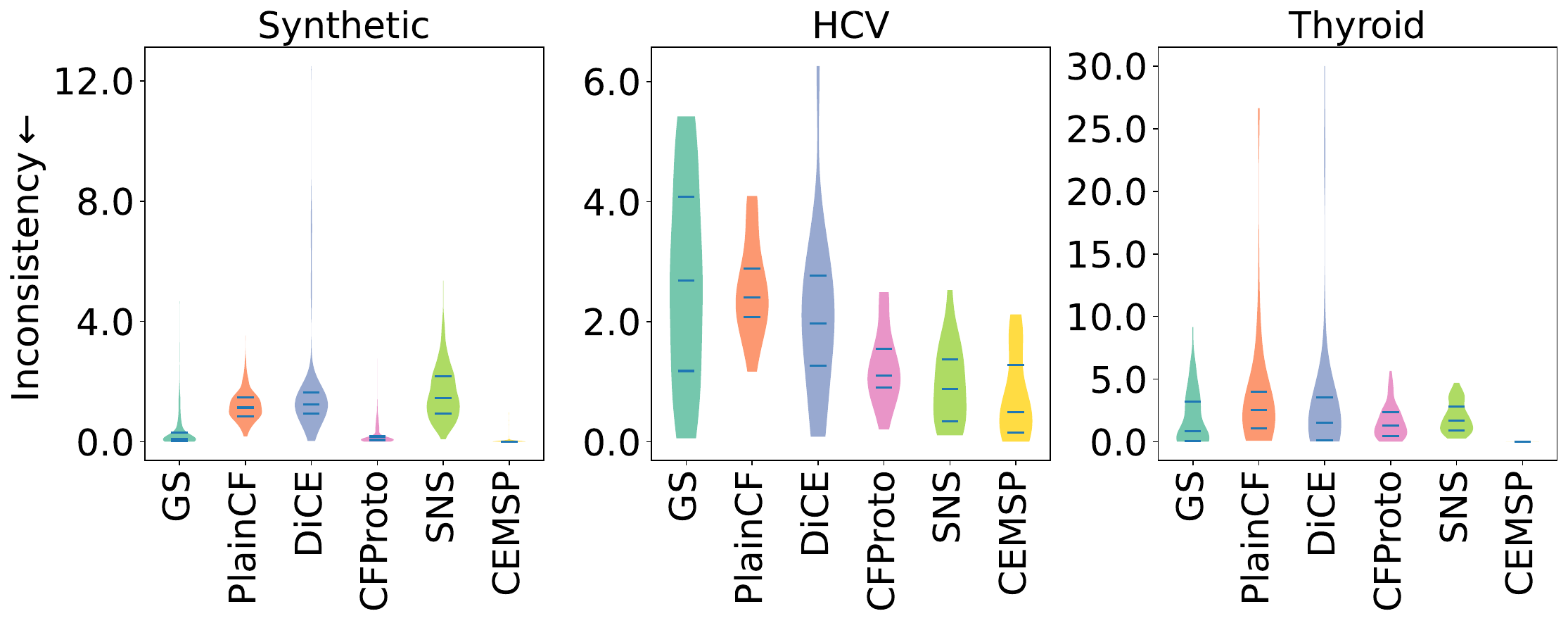}
    
    \caption{Evaluation of inconsistency score of model retraining.}
    \label{fig:inconsistency_s1}

    \includegraphics[width=0.95\linewidth]{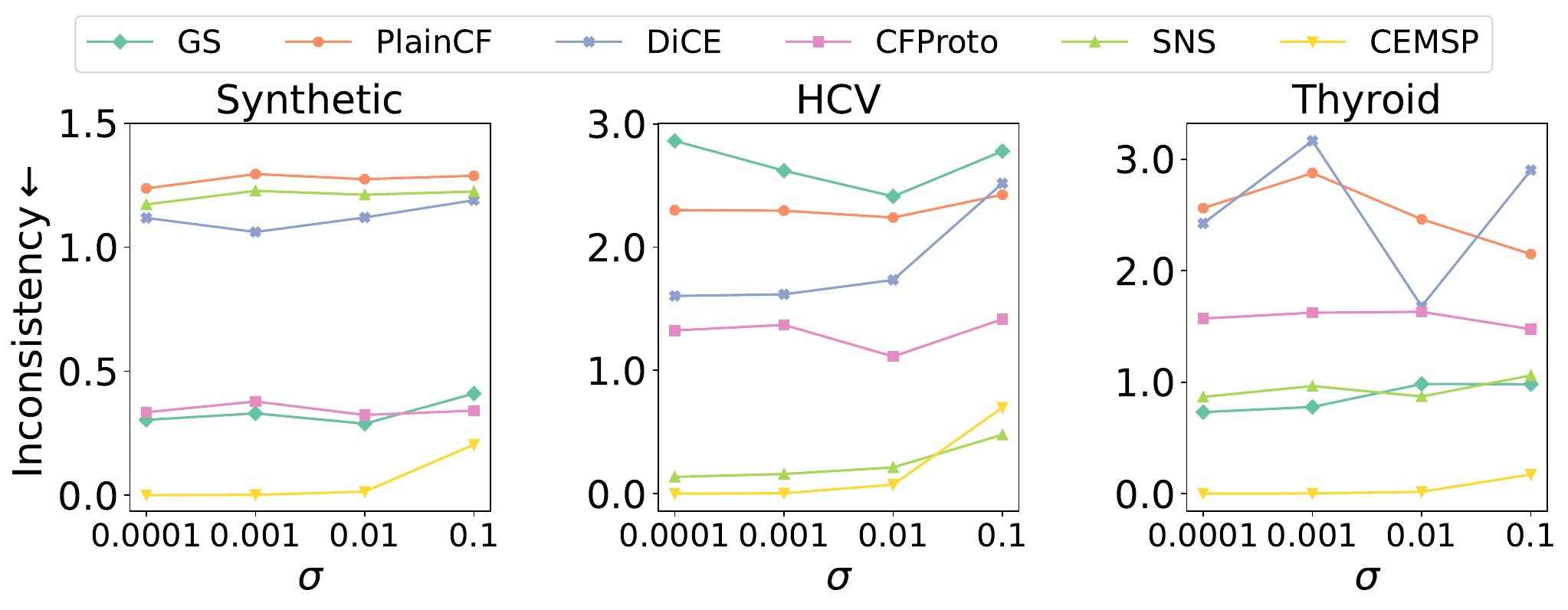}
    \caption{Evaluation of inconsistency score of input perturbations. $x$-axis represents the standard deviation $\sigma$ of added Gaussian noise.}
    \label{fig:inconsistency_s2}
\end{figure}

\subsection{Quantitative Evaluation}

We first report the quantitative comparison of sparsity, APS, Diversity, and C-Diversity in Figure \ref{fig:expscore}, where the $x$-axis denotes a counterfactual explanation generation method, the $y$-axis represents the average score of each metric of all evaluated instances. We can see that 
our method achieves the competitive sparsity as GS. GS achieves sparsity through post-preprocessing techniques, whereas our method \shortname\ focuses on making minimal modifications to subsets of abnormal features. In contrast, other methods do not explicitly optimize for sparsity and consequently fall behind in this aspect.
For the APS, CEMSP is slightly better. This result is grounded in two key considerations: firstly, we substitute an abnormal feature with its closest endpoints and secondly, we aim to change the minimal number of abnormal features. It is worth noting that although PlainCF minimizes the $L_1/MAD$ distance in its objective, this does not equate to minimizing the APS since APS is a density-aware metric among the population. Our method achieves at least $\frac{2}{d}$ C-Diversity as expected. In contrast to sparsity, C-Diversity sums up the fraction of features that are different between any two CFEs. Therefore, the method with a higher sparsity often associates with a lower C-Diversity. As a result, our CEMSP appears to have less competitive C-Diversity than methods that make simultaneous changes on many numerous features. However, our CEMSP has a competitive Diversity score defined in \cite{mothilal2020explaining}. 

Figure \ref{fig:inconsistency_s1} and \ref{fig:inconsistency_s2} report inconsistency scores under model retraining and input perturbations. Our CEMSP exhibits superior performance compared to other baseline methods. Specifically, GS, plainCF, and DiCE yield the poorest results when consistency restrictions are not enforced. CFProto, which incorporates a prototype term, achieves a better inconsistency score than PlainCF by directing all CFEs towards the prototype. As discussed earlier, our findings demonstrate that SNS does not perform well in generating CFEs with consistent feature values, despite having CFEs that yield consistent model predictions. In summary, our CEMSP outperforms the baseline methods across the aforementioned metrics, establishing its overall superiority.

\subsection{Use-Case Evaluation}
Next, we use the use-case evaluation in Figure \ref{fig:hcv_usecase} to present the compatibility of our method. The input instance is a patient in the HCV dataset who has the undesired prediction. Without any constraint, our method generates 4 CFEs, as illustrated in the table located at the bottom left. By introducing additional constraints, our model can effortlessly generate new CFEs that meet the desired criteria. For example, if we want to keep the original value of $BIL$, we can easily incorporate the CNF ($\lnot \mathbf{m}^5$), leading to CFEs that solely modify the remaining features. Furthermore, domain knowledge reveals that $ALT$ and $AST$ are correlated \cite{shen2015correlation}. Consequently, we can incorporate correlation constraints that limit simultaneous changes to both features. This can be achieved by including the CNF $(\lnot \mathbf{m}^3 \land \mathbf{m}^4) \vee (\mathbf{m}^3 \land \lnot \mathbf{m}^4)$ in our method, effectively enforcing the desired correlation constraint. Although this use-case evaluation may not yet be fully applicable to real-life scenarios, it offers valuable insights and demonstrates the potential to accommodate more practical considerations. 

\begin{figure}
    \centering
    \includegraphics[width=\linewidth]{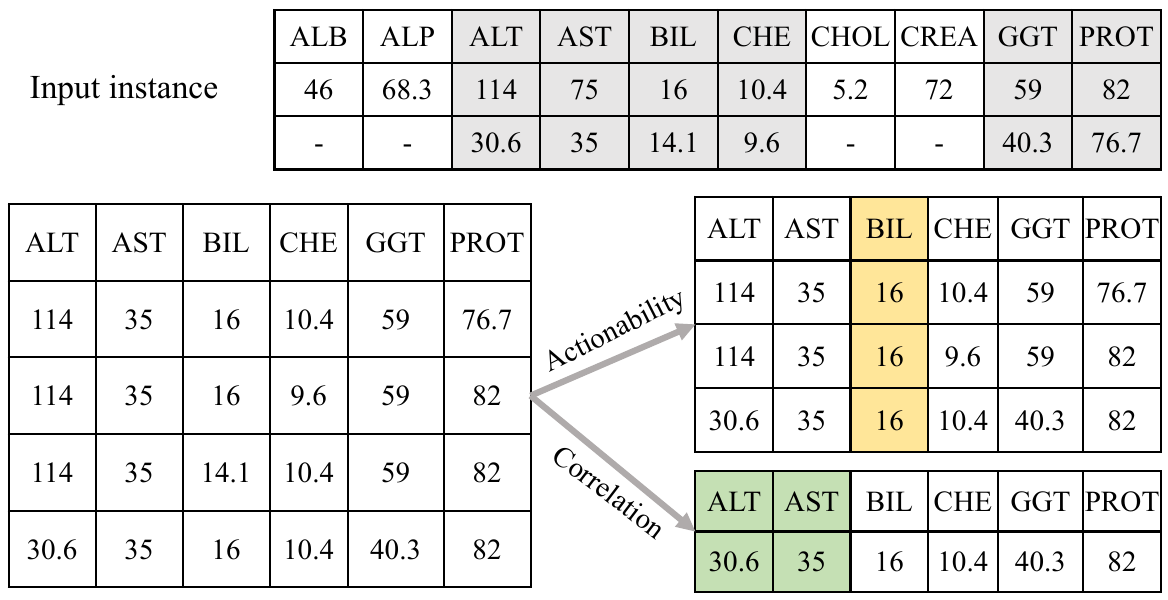}
    \caption{Use-case evaluation of a patient in the HCV dataset. The top table presents the original feature values of the patient, while the shaded features represent the altered features of CFEs and their closest endpoints. 
    The tables below display CFEs before/after incorporating constraints. To incorporate actionability and correlation, we introduced the expressions $\lnot m^5$ and $(\lnot \mathbf{m}^3 \land \mathbf{m}^4) \vee (\mathbf{m}^3 \land \lnot \mathbf{m}^4)$, respectively.}
    \label{fig:hcv_usecase}
\end{figure}

\section{Conclusion}
Lacking robustness in counterfactual explanations can undermine both individual fairness and model reliability. 
In this work, we present a novel framework to generate robust and diverse counterfactual explanations (CFEs).
Our work leverages the feature normal ranges from domain knowledge and generates CFEs that replace the minimal number of abnormal features to the closest endpoints of their normal ranges. We convert this problem into the Boolean satisfiability problem and solve it with modern SAT solvers. Experiments on both synthetic and real-life datasets demonstrate that our generated CFEs are more consistent than baselines while preserving flexibility for user preferences. 

\section{Limitations and Future Work}

While our work offers the potential to address the non-robustness issue through the utilization of domain knowledge, such as the normal ranges in healthcare and finance, some limitations hinder its applicability in broader contexts. Firstly, the scalability of the proposed method is underestimated. SAT solvers exhibit exponential complexity in the worst-case scenario. When dealing with a substantial number of features, the time required to find a binary mask from an SAT solver may surpass that of a forward pass in the DNN model. This concern can be addressed through empirical comparisons of the execution time between the SAT solver and DNN model revoking. Secondly, our approach is not directly applicable to scenarios where a portion of normal ranges is unknown. It might be necessary to incorporate additional information to determine the appropriate replacement values for these features. Thirdly, our study is established on binary classification tasks. However, the direct adaptation of our method to multi-class classification or regression tasks remains challenging. Normal ranges are typically contingent upon the target prediction. In the context of multi-class classification or regression, the target predictions can become intricate, rendering the normal ranges unattainable. 

In future work, our ultimate goal is to investigate robust and flexible counterfactual explanations in more general situations without any assumption about normal ranges. In addition, we intend to develop a sustainable system that offers users actionable recommendations and gathers valuable feedback to nourish continual enhancements. 

\begin{acks}
This research is supported, in part, by Alibaba Group through Alibaba Innovative Research (AIR) Program and Alibaba-NTU Singapore Joint Research Institute (JRI), Nanyang Technological University, Singapore. This research is also supported, in part, by the National Research Foundation, Prime Minister's Office, Singapore under its NRF Investigatorship Programme (NRFI Award No. NRF-NRFI05-2019-0002). Any opinions, findings and conclusions or recommendations expressed in this material are those of the authors and do not reflect the views of National Research Foundation, Singapore.
\end{acks}

\bibliographystyle{ACM-Reference-Format}
\balance
\bibliography{Ref}
\end{document}